\documentclass[journal]{IEEEtran}

\usepackage{float}
\usepackage{url}
\usepackage{graphicx}
\usepackage{mathtools, amsfonts, amsthm}
\usepackage{cuted} %
\usepackage{dsfont}
\usepackage{subdepth} %
\usepackage{float}
\usepackage{subcaption}
\usepackage{hyperref}

\usepackage[acronym, nowarn]{glossaries}
\glsdisablehyper %

\usepackage{tikz}
\usepackage{pgfplots}
\pgfplotsset{compat=newest}
\usetikzlibrary{external}
\usetikzlibrary{angles,quotes}

\usepackage{tabularray}
\UseTblrLibrary{booktabs}

\usepackage[ruled,vlined]{algorithm2e}%
\DontPrintSemicolon
\makeatletter
\newcommand{\removelatexerror}{\let\@latex@error\@gobble}
\makeatother

\usepackage[noabbrev, capitalise, sort]{cleveref}

\usepackage[backend=biber, style=ieee, isbn=false, doi=false, url=false]{biblatex}

\bibliography{references.bib}
\theoremstyle{plain}

\newtheorem{definition}{Definition}
\newtheorem{theorem}{Theorem}

\newlength{\figurewidth}
\newlength{\figureheight}

\definecolor{matlabblue}{rgb}{0 0.4470 0.741}
\definecolor{matlaborange}{rgb}{0.8500 0.3250 0.0980}
\definecolor{matlabyellow}{rgb}{0.9290 0.6940 0.1250}
\definecolor{matlabpurple}{rgb}{0.4940 0.1840 0.5560}
\definecolor{matlabgreen}{rgb}{0.4660 0.6740 0.1880}
\definecolor{matlablightblue}{rgb}{0.3010 0.7450 0.9330}
\definecolor{matlabred}{rgb}{0.6350 0.0780 0.1840}

\pgfplotscreateplotcyclelist{matlabcolor}{%
	matlabblue, mark=diamond*\\%
	matlaborange, mark=*\\%
	matlabyellow, mark=square*\\%
	matlabpurple, mark=triangle*\\%
	matlabgreen, mark=star\\%
	matlablightblue\\%
	matlabred\\%
}

\pgfplotscreateplotcyclelist{matlabcolor-3}{%
	matlabblue, mark=diamond*\\%
	matlaborange, mark=*\\%
	matlabyellow, mark=square*\\%
	matlabblue, mark=triangle*\\%
	matlaborange, mark=star\\%
	matlabyellow, mark=o\\%
}

\newcommand{\bone}{\boldsymbol{1}}
\newcommand{\ba}{\boldsymbol{a}}
\newcommand{\bb}{\boldsymbol{b}}

\newcommand{\bd}{\boldsymbol{d}}
\newcommand{\bp}{\boldsymbol{p}}

\newcommand{\bw}{\boldsymbol{w}}
\newcommand{\bx}{\boldsymbol{x}}
\newcommand{\by}{\boldsymbol{y}}
\newcommand{\bz}{\boldsymbol{z}}
\newcommand{\bA}{\boldsymbol{A}}

\newcommand{\bI}{\boldsymbol{I}}

\newcommand{\bS}{\boldsymbol{S}}

\newcommand{\bV}{\boldsymbol{V}}

\newcommand{\bLambda}{\boldsymbol{\Lambda}}

\newcommand{\bmu}{\boldsymbol{\mu}}
\newcommand{\bphi}{\boldsymbol{\phi}}

\newcommand{\bsigma}{\boldsymbol{\sigma}}

\newcommand{\bDelta}{\boldsymbol{\Delta}}

\DeclareMathOperator*{\argmax}{arg\,max}
\DeclareMathOperator*{\argmin}{arg\,min}

\newcommand{\expect}{\mathbb{E}}

\newacronym{sc}{SC}{Sensitivity Curve}
\newacronym{atc}{ATC}{Adapt-then-Combine}
\newacronym{ges}{GES}{gross-error sensitivity}
\newacronym{are}{ARE}{asymptotic relative efficiency}

\newacronym{mpe}{MPE}{Multivariate Power Exponential}

\newacronym{scm}{SCM}{Sensitivity Curve Maximization}
\newacronym{ascm}{ASCM}{Aligned Sensitivity Curve Maximization}
\newacronym{sascm}{SASCM}{Simplified Aligned Sensitivity Curve Maximization}
\newacronym{ios}{IOS}{Iterative Outlier Scissor}
\newacronym{alie}{ALIE}{'A Little Is Enough'}
\newacronym{scc}{SCC}{Self-Centered-Clipping}
\newacronym{faba}{FABA}{'Fast Aggregation algorithm against Byzantine Attacks'}
\newacronym{rfa}{RFA}{Robust Federated Aggregation}
\newacronym{rop}{ROP}{Relocated Orthogonal Pertupation}
\newacronym{bridge}{BRIDGE}{Byzantine-resilient decentralized gradient descent}
\newacronym{byrdie}{ByRDiE}{Byzantine-resilient distributed coordinate descent}
\newacronym{rtc}{RTC}{Remove-then-Clip}
\newacronym{ipm}{IPM}{Inner Product Manipulation}
\newacronym{lv}{LV}{Large Value}

\newacronym{mnist}{MNIST}{Modified National Institute of Standards and Technology}
\newacronym{cifar}{CIFAR}{Canadian Institute for Advanced Research}

\makeglossaries

\begin{document}
\title{Sensitivity Curve Maximization: Attacking Robust Aggregators in Distributed Learning}
\author{Christian A. Schroth, Stefan Vlaski and Abdelhak M. Zoubir
\author{Christian A. Schroth, Stefan Vlaski, \IEEEmembership{Member, IEEE}, Abdelhak M. Zoubir, \IEEEmembership{Life Fellow, IEEE}
	\thanks{C. A. Schroth and A. M. Zoubir have been funded by DFG project grant ZO 215/19-1.}
	\thanks{Christian A. Schroth and Abdelhak M. Zoubir are with the Signal Processing Group, Technische Universität Darmstadt, Germany, \{schroth, zoubir\}@spg.tu-darmstadt.de. }
	\thanks{Stefan Vlaski is with the Department of Electrical and Electronic Engineering, Imperial College London, SW7 2AZ London, U.K., s.vlaski@imperial.ac.uk.}
	\thanks{A preliminary version of this work appeared in the conference publication~\cite{Schroth.2023}.}
}}

\maketitle

\begin{abstract}
In distributed learning agents aim at collaboratively solving a global learning problem. It becomes more and more likely that individual agents are malicious or faulty with an increasing size of the network. This leads to a degeneration or complete breakdown of the learning process. Classical aggregation schemes are prone to breakdown at small contamination rates, therefore robust aggregation schemes are sought for. While robust aggregation schemes can generally tolerate larger contamination rates, many have been shown to be susceptible to carefully crafted malicious attacks. In this work, we show how the sensitivity curve (SC), a classical tool from robust statistics, can be used to systematically derive optimal attack patterns against arbitrary robust aggregators, in most cases rendering them ineffective. We show the effectiveness of the proposed attack in multiple simulations.
\end{abstract}


\section{Introduction}

\IEEEPARstart{D}{istributed} learning paradigms, such as federated or decentralized learning\footnote{We use “distributed” for any structure where data remains local at individual agents, which includes federated and decentralized architectures. The term “decentralized” is used for a network without fusion center. In the literature, these terms are sometimes used interchangeably.}, are an emerging technique to efficiently handle large amounts of data in dispersed locations. In cooperative networks, distributed algorithms can match the performance of centralized algorithms, with data access at a single location \cite{Vlaski.2023a, Sayed.2014, Lian.2017, Nassif.2020}. In practice, malicious (also called byzantine \cite{Lamport.1982}) agents may interfere with the efficient, but non-robust model aggregation step used in cooperative settings. Counteracting the effect of byzantine agents has given rise to byzantine robust distributed learning~\cite{Guerraoui.2024a}. 

There exist two popular distributed learning concepts: federated \cite{McMahan.2017,Rizk.2022, Sanchez.2024} and decentralized \cite{Vlaski.2021a} learning. In federated learning, all agents communicate with a central fusion center, which performs the model aggregation and distribution. In decentralized learning, agents communicate in a peer-to-peer fashion and each agent performs its own model aggregation.

There exist a large amount of robust aggregation methods for \emph{federated} learning. Most methods rely on a robust estimation of the mean \cite{Peng.2024}, e.g. the trimmed-mean, the median \cite{Yin.2018}, Krum \cite{Blanchard.2017} or the geometric median \cite{Pillutla.2022}. Other methods are based on the Huber loss \cite{Zhao.2024}, game theory \cite{Xie.2024}, randomization \cite{Ramezani-Kebrya.2022, Ozfatura.2024}, variance reduction \cite{Zhang.2023, Gorbunov.2023} or clipping \cite{Karimireddy.2021}. More methods can be found in the surveys \cite{Rodriguez-Barroso.2023, Moshawrab.2023, Khan.2023}. While not necessarily derived for decentralized settings, most of these schemes can be adapted to decentralized peer-to-peer aggregation.

The literature of tailor-made schemes for robust \emph{decentralized} optimization is more sparse. There exist methods based on the trimmed-mean, the median and Krum \cite{Sahoo.2022}, e.g., \gls{bridge} \cite{Yang.2019} and \gls{byrdie} \cite{Fang.2022}, based on clipping, e.g. \gls{rtc} \cite{Yang.2024a} and \gls{scc} \cite{He.2023} and based on the cosine similarity \cite{Ghavamipour.2024}. Other methods iteratively discard a certain amount of samples \cite{Wu.2023} or combine variance reduction with clipping \cite{Yu.2023}. A class of robust and efficient M-estimation based coordinate-wise robust aggregators are presented in \cite{Vlaski.2022, Schroth.2023a}. In \cite{Wang.2023a}, the network topology is robustified by constructing clusters based on the risk profiles of individual agents.

Attacks on robust aggregation schemes can be broadly classified into three categories. Firstly, naive attacks which do not take into account the underlying structure of the data or aggregation scheme, e.g. the addition of Gaussian noise, label flipping, sign flipping \cite{Khan.2023} or the addition of large values \cite{Schroth.2023, Ghavamipour.2024}. Most robust aggregation schemes are effective in defending against such attacks. But these attacks have the advantage that they do require only little additional information about the honest users. 
Secondly, there are more sophisticated attacks which try to exploit the underlying data or model, e.g. data and model poisoning attacks \cite{Fang.2020, Shejwalkar.2021}, the backdoor attack \cite{Wang.2020}, the echo attack \cite{Pasquini.2023, Ye.2024} or \gls{ipm} \cite{Xie.2020}. A very common attack in this category is \gls{alie}, which crafts a specific value, which is small enough to escape robust aggregation methods, but large enough to disturb the aggregation \cite{Baruch.2019}. These attack schemes require knowledge of the underlying data and model weights, but do not need to know the deployed aggregation scheme to effectively craft outliers.
Thirdly, there are tailored attacks which attack a specific robust aggregation scheme by exploiting its weaknesses, e.g. \gls{rop}, which is designed to circumvent centered clipping by estimating the center and injecting an update orthogonal to the honest update direction \cite{Ozfatura.2024}. These attack schemes are very effective against a certain aggregation types, but may fail when used against arbitrary aggregation schemes.
Surveys of different attack schemes can be found in \cite{Lyu.2020, Nair.2023}. 

To the best of our knowledge, there does not exist a design method which can provide a powerful attack against arbitrary robust aggregation schemes. The attack schemes listed above are either limited in effectiveness or versatility, hence, are not good candidates to develop a general attack framework. In this work, we propose a general attack design framework based on criteria from classical robust statistics, i.e. using methods which were designed to measure robustness to design a powerful attack.

There exist multiple well-known metrics to quantify the robustness of estimators, i.e. the influence function, the \gls{sc}, the breakdown point, the maximum-bias curve or the \gls{are} \cite{Zoubir.2012, Zoubir.2018, Maronna.2019}. In a prior work \cite{Schroth.2023}, we presented preliminary results to systematically analyze coordinate-wise robust aggregation methods. Therein, we used classical methods from robust statistics, such as the \gls{sc}, proposed by Tukey \cite{Tukey.1977, Huber.2002, Andrews.1972} for the study of the finite sample behavior of estimators. Based on the analysis of coordinate-wise robust aggregation methods, a new attack was proposed, which searches for the data point that maximizes the \gls{sc} and then injects this data point into the aggregation process. We call it \gls{scm}. It was shown that \gls{scm} breaks every analyzed coordinate-wise aggregation method. In this paper, we build upon this insight and extend \gls{scm} to multivariate aggregation schemes. We also ensure the alignment of the attack direction over multiple training rounds with \gls{ascm} and \gls{sascm}. Further, we adapt multiple robust aggregation schemes from federated learning to decentralized learning and analyze their respective \glspl{sc}. Our contributions are:
\begin{itemize}
	\item Based on the recently developed \gls{scm} attack, we develop novel \gls{ascm} and \gls{sascm} attacks, which combine the maximization of the \gls{sc} with the alignment of the attack over time.
	\item We perform a systematic analysis of the 1D and 2D-\glspl{sc} of various robust aggregation schemes and propose tailored \gls{scm}/\gls{sascm} attack schemes.
	\item We deploy and analyze multivariate M-estimators in robust aggregation.
	\item While we keep our arguments general to cover arbitrary decentralized topologies, when considering a fully connected graph, we also include federated learning.
\end{itemize}
The remainder of the paper is organized as follows. In \cref{sec:learning} decentralized learning and robust aggregation is introduced. The classical \gls{sc} and our extended definition is given in \cref{sec:sc}. \cref{sec:scm} derives and motivates the proposed \gls{ascm} and \gls{sascm} attacks. In \cref{sec:aggregators}, an overview of existing robust aggregation schemes and our adaption of multivariate M-estimators, including an analysis of all aggregation schemes using 1D- and 2D-\glspl{sc}, is given. Existing attack schemes are presented in \cref{sec:attacks} and new specific \gls{sascm} attacks in \cref{sec:example_attacks}. Finally, simulation results are shown in \cref{sec:simulations} and a conclusion and outlook are given in \cref{sec:conclusion}.

\emph{Notation:} Normal-font letters ($a$, $A$) denote a scalar, bold  lowercase ($\ba$) a vector and bold uppercase ($\bA$) a matrix;  calligraphic letters ($\mathcal{A}$) denote a set,  $\mathbb{R}$ denotes  the set of real numbers and $\mathbb{R}^{r \times 1}$, $\mathbb{R}^{r \times r}$ the set of column  vectors of size $r \times 1$, matrices of size $r \times r$, respectively; $\bA^{-1}$ is the matrix inverse; $\bA^{\top}$ is the matrix transpose; $|a|$ is the  absolute value of a scalar; $|\mathcal{A}|$ is the cardinality of a set; $\lVert \ba \rVert$ denotes the euclidean norm of a vector; $\cos\left( \ba, \bb\right) = (\ba^{\top} \bb) / (\lVert\ba\rVert \lVert\bb\rVert)$ denotes the cosine similarity.

\section{Decentralized Learning and Robust Aggregation}
\label{sec:learning}

In a general distributed learning problem, a collection of $K$ agents aims to collaboratively solve a stochastic optimization problem
\begin{equation}
	\bw^{\circ} = \argmin_{\bw \in \mathbb{R}^{r \times 1}} \frac{1}{K} \sum_{k=1}^{K} J_{k}(\bw)
\end{equation}
with the local objective function $J_{k}(\bw) = \expect[Q(\bw; \bx_{k})]$. Where $\bx_{k} \in \mathbb{R}^{r \times 1}$ denotes a random variable of dimension $r$ representing the data available at agent $k$ and $Q(\bw; \bx_{k})$ denotes the associated loss.

In decentralized learning, agents only exchange intermediate estimates on a peer-to-peer basis, without communicating with a central fusion center. For example, the \gls{atc} diffusion algorithm takes the form \cite{Sayed.2014, Vlaski.2022, Djuric.2018}
\begin{align}
	\bphi_{k,i} =&\: \bw_{k,i-1} - \mu \widehat{\nabla J}_{k}(\bw_{k,i-1})\label{eqn:adapt}\\
	\bw_{k,i} =&\: \sum_{\ell \in \mathcal{N}_{k}} a_{\ell k} \bphi_{\ell,i}
	\label{eqn:ATC}
\end{align} 
with the stochastic gradient approximation $\widehat{\nabla J}_{k}(\bw_{k,i-1})$, weights $a_{\ell k}$, step-size $\mu > 0$ and the closed neighborhood of agent $k$ by $\mathcal{N}_{k}$, where closed indicates that agent $k$ is included in $\mathcal{N}_{k}$. Commonly, $\widehat{\nabla J}_{k}(\bw_{k,i-1}) = \nabla Q(\bw_{k, i-1}; \bx_{k, i})$ is chosen with $\bx_{k, i}$ denoting the sample available at agent $k$ at time $i$. The graph is described by a left-stochastic weight matrix $\bA \in \mathbb{R}^{K \times K}$ with entries 
\begin{equation}
	a_{\ell k} \begin{cases}
		> 0, &\text{if } \ell \in \mathcal{N}_{k} \\
		= 0, &\text{else}
	\end{cases}
\end{equation}
and
\begin{equation}
	\sum_{\ell \in \mathcal{N}_{k} } a_{\ell k} = 1.
\end{equation}
Two common combination rules \cite[p.~77]{Djuric.2018} are the uniform averaging rule
\begin{equation}
	a_{\ell k} = \frac{1}{|\mathcal{N}_{k}|}, \quad \ell \in \mathcal{N}_{k}
	\label{eqn:uni_avg}
\end{equation}
which leads to a left-stochastic weight matrix $\bA$ and the Metropolis rule
\begin{equation}
	a_{\ell k} = \begin{cases}
		\frac{1}{\max\left(|\mathcal{N}_{k}|, |\mathcal{N}_{\ell}|\right)}, & \ell \in \mathcal{N}_{k} \backslash \{k\} \\
		1 - \sum\limits_{m \in \mathcal{N}_{k} \backslash {k}} a_{mk}, & \ell = k
	\end{cases}
	\label{eqn:metropolis}
\end{equation}
which leads to a doubly-stochastic weight matrix $\bA$. In collaborative scenarios without any byzantine agents, the choice of the above combination policies is well justified. This is because they lead to sufficiently high efficiency and fast convergence rates. For example, in a scenario without outliers, the mean can be observed to converge fastest as demonstrated in \cref{fig:sim_no_out_1}. In non-collaborative scenarios, the agents in the network are split into a honest $\mathcal{H}$ and a byzantine $\mathcal{B}$ subset. Similarly, the neighborhood $\mathcal{N}_{k} = \mathcal{H}_{k}\cup \mathcal{B}_{k}$ of agent $k$ is split into a honest $\mathcal{H}_{k}$ (including agent $k$) and a byzantine $\mathcal{B}_{k}$ neighborhood. The aggregating agent $k$ receives from its neighbors the intermediate update
\begin{align} \bphi_{\ell,i} = 
	\begin{cases}
		\bphi_{\ell,i}, & \ell \in \mathcal{H}_{k} \\
		*, & \ell \in \mathcal{B}_{k},
	\end{cases}
\end{align} 
where $*$ stands for an arbitrarily and maliciously crafted sample from a byzantine agent. As the averaging-based aggregation schemes are highly susceptible to this kind of malicious samples, the non-robust combination rule in Equation~\eqref{eqn:ATC} has to be replaced with a general, preferably robust, aggregation rule
\begin{align}
	\bw_{k,i} = \mathbf{agg}\left(\{\bphi_{\ell,i}\}_{\ell \in \mathcal{N}_{k}}\right) \in \mathbb{R}^{r \times 1}.
	\label{eqn:robAGG}
\end{align}
This robust aggregation rule is able to reduce or completely eliminate the influence of malicious samples. The complete decentralized learning procedure with \gls{atc}-diffusion and robust aggregation is summarized in \cref{alg:dec_sgd}.

\begin{figure}[t]
	\removelatexerror
	\begin{algorithm}[H]
		Initialize $\{\bw_{k,0}\}_{k \in \mathcal{H}}$ \;
		\For{$i = 1,2,\dots$}
		{%
			\For{$k \in \mathcal{H}$}
			{%
				Compute local gradient and update local weight\;
				\[\bphi_{k,i} = \bw_{k,i-1} - \mu \widehat{\nabla J}_{k}(\bw_{k,i-1})\]
				Send weight $\bphi_{k,i}$ to $\mathcal{N}_{k}$ \;
			}
			\For{$k \in \mathcal{B}$}
			{%
				Craft malicious weight $\bphi_{k,i} = *$\;
				Send malicious weights to all agents $\ell$, such that $k \in \mathcal{N}_{\ell}$ \;
			}
			\For{$k \in \mathcal{H}$}
			{%
				Receive weights $\{\bphi_{\ell,i}\}_{\ell \in \mathcal{N}_{k}}$ \;
				Robustly aggregate received weights\;
				\[\bw_{k,i} = \mathbf{agg}\left(\{\bphi_{\ell,i}\}_{\ell \in \mathcal{N}_{k}}\right)\]
			}
		}
		\caption{ATC diffusion with robust aggregation}
		\label{alg:dec_sgd}
	\end{algorithm}
\end{figure}

\section{Sensitivity Curve}
\label{sec:sc}
The \gls{sc} was proposed by Tukey \cite{Tukey.1977, Huber.2002, Andrews.1972} to study the finite sample behavior of estimators. It can also be used to measure the impact of an outlier on the aggregation result. To simplify the notation, we denote the weights of the honest agents $\{\bphi_{\ell,i}\}_{\ell \in \mathcal{H}_{k}}$ as $\mathcal{Y}$ ($\mathcal{Y}_{i}$ at time $i$) and the weights of the byzantine agents $\{\bphi_{\ell,i}\}_{\ell \in \mathcal{B}_{k}}$ as $\mathcal{Z}$ ($\mathcal{Z}_{i}$ at time $i$).

For an aggregator $\mathbf{agg}(\cdot)$ and a set of scalar samples $\mathcal{Y} = \{y_{1}, \dots, y_{N-1}\}$ of size $N-1$, the \gls{sc} is defined~as
\begin{equation}
	\text{sc}(\mathcal{Y}, z) = N \left(\mathbf{agg}(\mathcal{Y} \cup z) - \mathbf{agg}(\mathcal{Y})\right) \in \mathbb{R}^{1 \times 1},
\end{equation}
where it describes the sensitivity of an aggregator $\mathbf{agg}(\cdot)$ to an additional sample $z$. It has to be noted that the value of the \gls{sc} depends on the underlying sample $\mathcal{Y}$. Thus, the value of the \gls{sc} fluctuates for different samples $\mathcal{Y}$.

We extend the definition of the \gls{sc} to account for $P$ identical multivariate outliers $\mathcal{Z} = \{\bz, \dots, \bz\}$, where $\bz \in \mathbb{R}^{r \times 1}$ is repeated $P$-times in $\mathcal{Z}$. The uncontaminated sample becomes $\mathcal{Y} = \{\by_{1}, \dots, \by_{N-P}\}$, $\by_{n} \in \mathbb{R}^{r \times 1}$ and the multivariate \gls{sc} becomes
\begin{equation}
	\mathbf{sc}(\mathcal{Y}, \mathcal{Z}) = N \left(\mathbf{agg}(\mathcal{Y} \cup \mathcal{Z})- \mathbf{agg}(\mathcal{Y})\right) \in \mathbb{R}^{r \times 1}
	\label{eqn:sc}
\end{equation}
Although simplified to multiple identical multivariate outliers, this definition allows for a good tractability and is sufficiently flexible to yield strong attacks, as we will show later.

Some properties of the \gls{sc} are depicted in \cref{fig:GES}, which can be directly translated into important robustness measures \cite[Chapter~2.1c]{Hampel.1986}. The \gls{ges} describes the worst influence, which a set of outliers can have on the result of the aggregation. To achieve robustness the \gls{ges} has to be finite and preferably small. As pointed out in \cite{Hampel.1986}, decreasing the \gls{ges} will decrease the efficiency and vice versa. Hence, the aggregator with the highest efficiency (sample mean) has an infinite \gls{ges} and the aggregator with the lowest \gls{ges} has a zero efficiency (no aggregation). The classical aggregator with the lowest non-zero \gls{ges} is the median \cite{Hampel.1986}. There also exists an optimal redescending M-estimator which minimizes the \gls{ges} while bounding the asymptotic variance \cite[p.~153]{Maronna.2019}. The local-shift sensitivity describes the effect of small fluctuations in the samples. A small change in the samples should also lead to a small change in the \gls{sc}. Therefore, the \gls{sc} should be a continuous function. Lastly, the rejection point describes the point, where a sample is completely rejected from the aggregation as it is too distinct from the norm. Preferably, the rejection point should be finite. The 1D-\glspl{sc} of various aggregators are depicted in \cref{fig:SC,fig:SC2} and the 2D-\glspl{sc} in \cref{fig:2d-sc-1,fig:2d-sc-2,fig:2d-sc-3}.

\begin{figure}[t]
	\centering
	\resizebox{0.8\columnwidth}{!}{
	\includegraphics{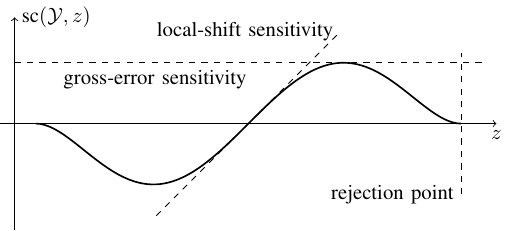}}
	\caption{Properties of the \gls{sc} for $r=1$ (adapted from \cite{Hampel.1986}).}
	\label{fig:GES}
\end{figure}

\section{Optimal Attack Pattern}
\label{sec:scm}
Injecting a well crafted outlier has been shown to be an effective method to circumvent various robust aggregation schemes. The following methods fall under this category. In \cite{Baruch.2019}, the authors propose \gls{alie}, which crafts an outlier that is small enough to not be rejected as an outlier, but large enough to influence the aggregation result. In the MIMIC attack in \cite{Karimireddy.2022}, the attacker copies the sample from a benign agent to introduce a bias towards this agent. In the Gaussian attack \cite{Ghavamipour.2024}, the attacker sends a sample randomly drawn from a Gaussian distribution.

These attacks have in common that they are based on heuristics which seem to work well, but might not be optimal. To be able to provide new justifications for the effectiveness of existing attack schemes, and systematically develop new attacks for arbitrary aggregation schemes, we propose the following definition of an optimal attack:
\begin{definition}[Optimal Attack]
	For a given aggregator $\mathbf{agg}(\cdot)$ and a given sample $\mathcal{Y}$, the sample $\mathcal{Z}^{\star}$ which maximizes the distance $\mathrm{dist}(\cdot)$ between the attacked aggregation result $\mathbf{agg}(\mathcal{Y} \cup \mathcal{Z})$ and the benign aggregation result $\mathbf{agg}(\mathcal{Y})$, can be found by solving
	\begin{equation}
		\mathcal{Z}^{\star} = \argmax_{\mathcal{Z}}\left( \mathrm{dist}(\mathbf{agg}(\mathcal{Y} \cup \mathcal{Z}), \mathbf{agg}(\mathcal{Y}))\right).
	\end{equation}
	\label{def:opt_attack}
\end{definition}
Intuitively, we want to find the byzantine sample, which maximizes the distance between the benign aggregation result and the attacked aggregation result. Choosing the squared Euclidean distance $\mathrm{dist}(\ba, \bb) = \lVert\ba-\bb\rVert^{2}$ and restricting the set of outliers $\mathcal{Z}$ to only contain $P$ identical outliers, a new attack scheme, which we will call \glsreset{scm}\gls{scm}, can be found:
\begin{theorem}[Sensitivity Curve Maximization]
	For a given aggregator $\mathbf{agg}(\cdot)$ and a given sample $\mathcal{Y}$, the sample $\mathcal{Z}^{\star} = \{\bz^{\star}, \dots, \bz^{\star}\}$  with $P$ identical samples, which maximizes the distance $\mathrm{dist}(\mathbf{agg}(\mathcal{Y} \cup \mathcal{Z}), \mathbf{agg}(\mathcal{Y}))$ for $\mathrm{dist}(\ba, \bb) = \lVert\ba-\bb\rVert^{2}$, can be found by solving
	\begin{equation}
		\mathcal{Z}^{\star} = \argmax_{\mathcal{Z}} \left\lVert\mathbf{sc}(\mathcal{Y}, \mathcal{Z})\right\rVert^{2}.
	\end{equation}
	\label{th:SCM}
\end{theorem}
\begin{proof}
Using the distance measure $\mathrm{dist}(\ba, \bb) = \lVert\ba-\bb\rVert^{2}$ and a set $\mathcal{Z}$ which contains $P$ identical outliers, we find
\begin{align}
	\mathrm{dist}&(\mathbf{agg}(\mathcal{Y} \cup \mathcal{Z}), \mathbf{agg}(\mathcal{Y})) \nonumber\\
	=& \lVert\mathbf{agg}(\mathcal{Y} \cup \mathcal{Z}) - \mathbf{agg}(\mathcal{Y})\rVert^{2} 
	\overset{\eqref{eqn:sc}}{=} \frac{1}{N^{2}}\left\lVert\mathbf{sc}(\mathcal{Y}, \mathcal{Z})\right\rVert^{2},
\end{align}
omitting the factor $\frac{1}{N^{2}}$, we obtain \cref{th:SCM}.
\end{proof}

We now illustrate how several of the effective attack schemes proposed in the literature can be developed in a unified manner using \cref{def:opt_attack} and \cref{th:SCM}. For example, the \gls{lv} attack maximizes the \gls{sc} of the Sample Mean over a bounded set, by selecting an arbitrary large value, as the \gls{sc} of the Sample Mean does linearly increase with the outlier value. \gls{alie} on the other hand does not directly maximize the \gls{sc}, but uses the normal distribution as a surrogate function, to find a value which corresponds to a large \gls{sc} value. More details on these attacks can be found in \cref{sec:attacks}.

For most aggregation schemes there exist multiple samples $\mathcal{Z}^{\star}$; observe for example the multiple maxima in \cref{fig:2d-sc-1,fig:2d-sc-3,fig:2d-sc-2}. The authors of \cite{Ozfatura.2024} show that it is not feasible to randomly select a direction for a given distance, even if it is a solution of \cref{th:SCM}, as this will reduce the impact of the attack. This can be contributed to the fact that attacks in random directions might cancel out over multiple learning rounds, as demonstrated in \cref{fig:sim_ones}. Therefore it is important to introduce a temporal component into the \gls{scm} attack, which ensures an accumulation of the attack over multiple time steps. In \cref{fig:multi_update}, two consecutive learning rounds are shown, where the red dashed lines represent the \gls{sc}. As the \glspl{sc} in this figure are not aligned, a slight cancellation of the effect of the \glspl{sc} can be observed. To avoid this cancellation and to maximize the effect of the attack over multiple time steps, we introduce:
\begin{theorem}[Aligned Sensitivity Curve Maximization]
	For a given previous $\mathbf{sc}(\mathcal{Y}_{i-1}, \mathcal{Z}^{\star}_{i-1})$ from time $i-1$ and a current $\mathbf{sc}(\mathcal{Y}_{i}, \mathcal{Z}_{i})$ from time $i$, the optimal outlier sample $\mathcal{Z}_{i}^{\star}$ which maximizes an upper bound of the malicious divergence over two consecutive time steps, can be found by solving
	\begin{align}
		\mathcal{Z}_{i}^{\star} = &\argmax_{\mathcal{Z}_{i}} \bigl\{  \left\lVert \mathbf{sc}(\mathcal{Y}_{i}, \mathcal{Z}_{i}) \right\rVert^{2}  + 2 \left\lVert \mathbf{sc}(\mathcal{Y}_{i-1}, \mathcal{Z}^{\star}_{i-1}) \right\rVert \nonumber\\ 
		& \cdot \left\lVert \mathbf{sc}(\mathcal{Y}_{i}, \mathcal{Z}_{i}) \right\rVert \cos\left( \mathbf{sc}(\mathcal{Y}_{i-1}, \mathcal{Z}^{\star}_{i-1}), \mathbf{sc}(\mathcal{Y}_{i}, \mathcal{Z}_{i}) \right) \bigr\}
	\end{align}
	with an initial optimal outlier sample
	\begin{equation}
		\mathcal{Z}_{0}^{\star} = \argmax_{\mathcal{Z}_{0}} \left\lVert\mathbf{sc}(\mathcal{Y}_{0}, \mathcal{Z}_{0})\right\rVert^{2}.
	\end{equation}
	\label{th:ascm}
\end{theorem}

\begin{proof}
	The distance to be maximized is denoted with the blue dash-dotted line in \cref{fig:multi_update}, i.e. 
	\begin{align}
		& \lVert (\bw_{k,i-1} -\bar{\bw}_{k,i-1}) -\mu \widehat{\nabla J}_{k}(\bw_{k,i-1}) \nonumber \\
		&+ \bar{\bw}_{k,i} - \bphi_{k,i} + (\bw_{k,i} -\bar{\bw}_{k,i}) \rVert^{2}
	\end{align}
	with $\bar{\bw}_{k,i} = \mathbf{agg}\left(\{\bphi_{\ell,i}\}_{\ell \in \mathcal{H}_{k}}\right)$ and letting $\mu \rightarrow 0$, we can neglect the effect of the gradient update, approximate $\bphi_{k,i} \approx \bw_{k,i-1}$ and we simplify to
	\begin{align}
		\lVert (\bw_{k,i-1} -\bar{\bw}_{k,i-1}) + \bar{\bw}_{k,i} - \bw_{k,i-1} + (\bw_{k,i} -\bar{\bw}_{k,i}) \rVert^{2}.
	\end{align}
	Using the previously introduced notation and dropping the node index $k$, the maximization problem becomes
	\begin{align}
	 \max_{\mathcal{Z}_{i}}& \bigl\{\lVert \mathbf{sc}(\mathcal{Y}_{i-1}, \mathcal{Z}^{\star}_{i-1})  + \mathbf{sc}(\mathcal{Y}_{i}, \mathcal{Z}_{i}) \nonumber \\
	 & + \mathbf{agg}(\mathcal{Y}_{i}) - \mathbf{agg}(\mathcal{Y}_{i-1} \cup \mathcal{Z}^{\star}_{i-1})\rVert^{2}\bigr\} \nonumber\\
	\leq\max_{\mathcal{Z}_{i}}& \bigl\{2 \lVert \mathbf{sc}(\mathcal{Y}_{i-1}, \mathcal{Z}^{\star}_{i-1}) + \mathbf{sc}(\mathcal{Y}_{i}, \mathcal{Z}_{i}) \rVert^{2}  \nonumber\\
	 & + 2 \lVert \mathbf{agg}(\mathcal{Y}_{i}) - \mathbf{agg}(\mathcal{Y}_{i-1} \cup \mathcal{Z}^{\star}_{i-1})  \rVert^{2}\bigr\} \nonumber\\
	=\max_{\mathcal{Z}_{i}}& \bigl\{2 \lVert \mathbf{sc}(\mathcal{Y}_{i-1}, \mathcal{Z}^{\star}_{i-1}) \rVert^{2} + 2 \lVert \mathbf{sc}(\mathcal{Y}_{i}, \mathcal{Z}_{i})  \rVert^{2} \nonumber\\
	 & + 2 \lVert \mathbf{agg}(\mathcal{Y}_{i}) - \mathbf{agg}(\mathcal{Y}_{i-1} \cup \mathcal{Z}^{\star}_{i-1})  \rVert^{2} \nonumber\\
	 & + 4 \lVert \mathbf{sc}(\mathcal{Y}_{i-1}, \mathcal{Z}^{\star}_{i-1}) \rVert \cdot \lVert \mathbf{sc}(\mathcal{Y}_{i}, \mathcal{Z}_{i}) \rVert \nonumber\\
	 & \cdot \cos \left(\mathbf{sc}(\mathcal{Y}_{i-1}, \mathcal{Z}^{\star}_{i-1}), \mathbf{sc}(\mathcal{Y}_{i}, \mathcal{Z}_{i})\right) \bigr\}
	\end{align}
	removing terms which are constant with respect to the maximization and solving for the argument leads to
	\begin{align}
		\mathcal{Z}^{\star}_{i} =& \argmax_{\mathcal{Z}_{i}} \bigl\{ \lVert \mathbf{sc}(\mathcal{Y}_{i}, \mathcal{Z}_{i}) \rVert^{2} + 2 \lVert \mathbf{sc}(\mathcal{Y}_{i-1}, \mathcal{Z}^{\star}_{i-1}) \rVert \nonumber \\
		& \cdot \lVert \mathbf{sc}(\mathcal{Y}_{i}, \mathcal{Z}_{i}) \rVert \cos \left(\mathbf{sc}(\mathcal{Y}_{i-1}, \mathcal{Z}^{\star}_{i-1}), \mathbf{sc}(\mathcal{Y}_{i}, \mathcal{Z}_{i})\right) \bigr\}
	\end{align}
	which proves \cref{th:ascm}.
\end{proof}

\begin{figure}
	\centering
	\hfil
	\resizebox{1\columnwidth}{!}{
		\begin{tikzpicture}[scale = 2]
			
			\coordinate (w_old) at (0.5, 0.5);
			\coordinate (phi_k) at (2, 2.5);
			\coordinate (w_benign) at (1, 2);
			\coordinate (w) at (3, 1);
			\coordinate (phi_k_i_1) at (4.5, 2);
			\coordinate (w_benign_i_1) at (5.7, 3);
			\coordinate (w_i_1) at (4.2, 2.9);
			
			\draw[->] (w_old) node[below] {$\bw_{k,i-2}$} -- node[left] {$-\mu \widehat{\nabla J}_{k}(\bw_{k,i-2})$} (phi_k) node[above] {$\bphi_{k,i-1}$};
			\draw[->] (phi_k) -- node[right] {actual update} (w) ;
			\draw[->] (phi_k) --  node[left] {honest update} (w_benign) node[below] {$\bar{\bw}_{k,i-1}$};
			\draw[dashed,red, <-, thick] (w) -- (w_benign);
			
			\draw[->] (w) node[below] {$\bw_{k,i-1}$} -- node[right] {$-\mu \widehat{\nabla J}_{k}(\bw_{k,i-1})$} (phi_k_i_1) node[right] {$\bphi_{k,i}$};
			\draw[->] (phi_k_i_1) -- node[left] {actual update} (w_i_1) node[left] {$\bw_{k,i}$};;
			\draw[->] (phi_k_i_1) --  node[right] {honest update} (w_benign_i_1) node[right] {$\bar{\bw}_{k,i}$};
			
			\draw[dashed,red, <-, thick] (w_i_1) -- (w_benign_i_1);
			\draw[dash dot,blue, thick, ->] (w_benign) -- (w_i_1);
		\end{tikzpicture}
	}
	\hfil
	\caption{Multiple rounds of learning at agent $k$ without aligned \gls{scm}. The red dashed line indicates the \gls{sc} and $\bar{\bw}_{k,i} = \mathbf{agg}\left(\{\bphi_{\ell,i}\}_{\ell \in \mathcal{H}_{k}}\right)$ indicates the honest aggregation result.}
	\label{fig:multi_update}
\end{figure}
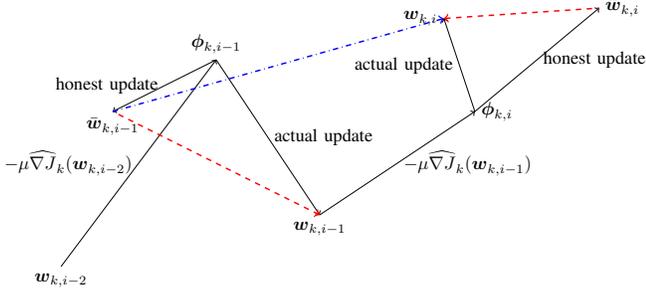

With the definition of the \gls{ascm} from \cref{th:SCM}, we are now able to determine an optimal attack sample $\mathcal{Z}^{\star}_{i}$ at time $i$, which takes into account the optimal attack sample from the previous time $\mathcal{Z}^{\star}_{i-1}$ and  maximizes the effect of the attack over both time steps. In \cref{th:ascm}, we maximize the norm of the current \gls{sc} with the cosine similarity between the current and the previous \gls{sc}. This motivates the following definition of a simplified attack, which we show later to be effective in practice:
\begin{definition}[Simplified Aligned Sensitivity Curve Maximization]
	The optimal outlier sample $\mathcal{Z}_{i}^{\star}$ which maximizes the current $\mathbf{sc}(\mathcal{Y}_{i}, \mathcal{Z}_{i})$, while keeping it aligned with the previous $\mathbf{sc}(\mathcal{Y}_{i-1}, \mathcal{Z}^{\star}_{i-1})$, can be found by solving	
	\begin{align}
		\mathcal{Z}_{i}^{\star} = & \argmax_{\mathcal{Z}_{i}} \left\lVert\mathbf{sc}(\mathcal{Y}_{i}, \mathcal{Z}_{i})\right\rVert^{2} \nonumber \\
		& \text{subject to} \nonumber\\
		&\cos\left( \mathbf{sc}(\mathcal{Y}_{i}, \mathcal{Z}_{i}), \mathbf{sc}(\mathcal{Y}_{i-1}, \mathcal{Z}^{\star}_{i-1})\right) \geq \delta,
	\end{align}
	where $\delta$ is a constant, which should be chosen close to $1$.
	\label{def:sascm}
\end{definition}
This simplified attack allows us to independently maximize the strength of the attack while keeping the attack aligned with the previous attack.

While \cref{th:SCM} provides a general rule to find an attack value which has the maximum effect on the aggregation at a certain individual time step, \cref{th:ascm} enables us to ensure a maximum accumulation of the attacks over time. Finally, \cref{def:sascm} provides a good trade-off between ease of use and a powerful attack design.

\section{Aggregation Schemes}
\label{sec:aggregators}

This section analyzes and gives an overview of aggregation schemes which can be used in place of the general aggregation function $\mathbf{agg}(\cdot)$ in Equation~\eqref{eqn:robAGG}. To objectively measure the influence of an outlier on the aggregation result, we will use the \gls{sc}. The classical 1D-\glspl{sc} are presented in \cref{fig:SC,fig:SC2}. Properties which are favorable to observe in these figures, are a finite \gls{ges} and a finite rejection point, combined with a linear local shift sensitivity, as explained in \cref{sec:sc}. These three properties are well-known to be critical from the perspective of classical robust statistics \cite{Hampel.1986}. Therefore it is not surprising that these properties emerge naturally in modern robust aggregators for distributed learning, like in \gls{ios} or Multi-Krum.

As most robust aggregators are multivariate aggregators, it is not sufficient to analyze their 1D-\glspl{sc}. Therefore the 2D-\glspl{sc}, based on Equation~\eqref{eqn:sc}, are also analyzed and depicted in \cref{fig:2d-sc-1,fig:2d-sc-3,fig:2d-sc-2}.

\begin{figure}
	\centering
	\hfil
	\subcaptionbox{\glspl{sc} of sample mean, $\alpha$-trimmed mean, median and \gls{scc} \label{fig:sc1}}[.49\columnwidth]{\resizebox{0.49\columnwidth}{!}{\input{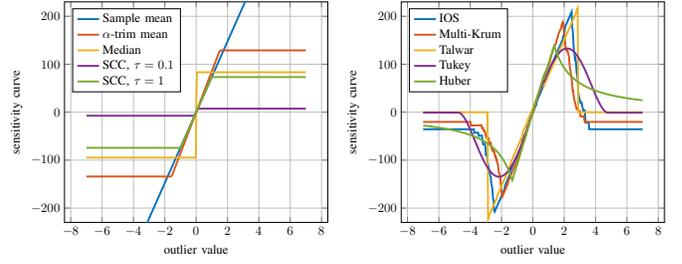}}}
	\hfil
	\subcaptionbox{\glspl{sc} of \gls{ios}, Multi-Krum, Talwar, Tukey and Huber
		\label{fig:sc2}}[.49\columnwidth]{\resizebox{0.49\columnwidth}{!}{\input{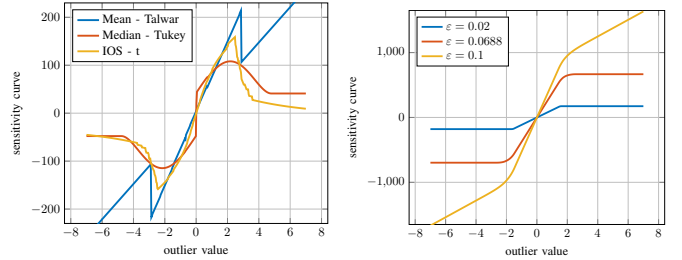}}}
	\hfil
	\caption{Overview of \glspl{sc} for different aggregation schemes for $r=1$.}
	\label{fig:SC}
\end{figure}

\begin{figure}
	\centering
	\hfil
	\subcaptionbox{\glspl{sc} of exemplary MixTailor aggregators. \label{fig:sc-mix}}[.49\columnwidth]{\resizebox{0.49\columnwidth}{!}{\input{figs/SC-mixtailor.tex}}}
	\hfil
	\subcaptionbox{\glspl{sc} of $\alpha$-trimmed mean with different contamination rates and $\alpha=0.0688$.
		\label{fig:sc-trimm}}[.49\columnwidth]{\resizebox{0.49\columnwidth}{!}{\input{figs/SC-trimm.tex}}}
	\hfil
	\caption{Overview of \glspl{sc} for different aggregation schemes for $r=1$.}
	\label{fig:SC2}
\end{figure}

\begin{figure}
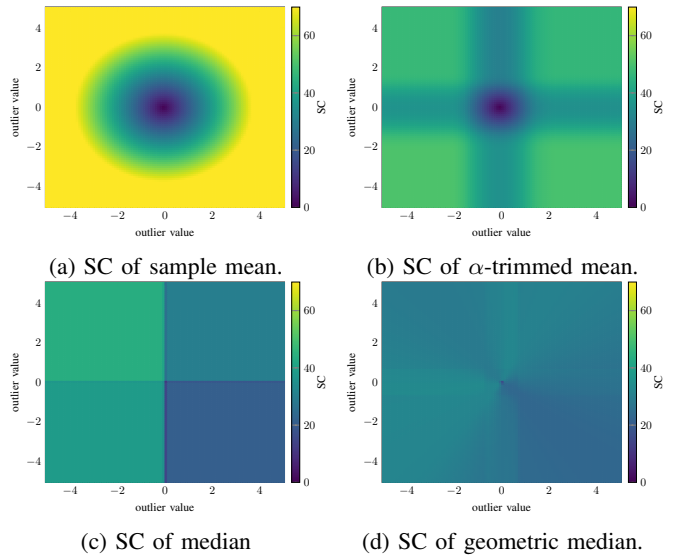

	\centering
	\hfil
	\subcaptionbox{\gls{sc} of sample mean.
		\label{fig:2d-sc-mean}}[.49\columnwidth]{\resizebox{0.49\columnwidth}{!}{\input{figs/2D-SC-Mean.tex}}}
	\hfil
	\subcaptionbox{\gls{sc} of $\alpha$-trimmed mean.
		\label{fig:2d-sc-trimm}}[.49\columnwidth]{\resizebox{0.49\columnwidth}{!}{\input{figs/2D-SC-TriMean.tex}}}
	\hfil \\
	\hfil
	\subcaptionbox{\gls{sc} of median
		\label{fig:2d-sc-median}}[.49\columnwidth]{\resizebox{0.49\columnwidth}{!}{\input{figs/2D-SC-Median.tex}}}
	\hfil
	\subcaptionbox{\gls{sc} of geometric median.
		\label{fig:2d-sc-geomedian}}[.49\columnwidth]{\resizebox{0.49\columnwidth}{!}{\input{figs/2D-SC-GeoMedian.tex}}}
	\hfil 	
	\caption{Euclidean norm of 2D-\glspl{sc} for different aggregation schemes. Values larger than 70 are clipped. Mean of underlying data is at $(0,0)$.}
	\label{fig:2d-sc-1}
\end{figure}

\emph{Sample Mean:} Classical aggregation scheme \cite{McMahan.2017, Tang.2018}, obtained using the uniform averaging rule from Equation~\eqref{eqn:uni_avg} in Equation~\eqref{eqn:ATC} as
\begin{equation}
	\bw_{k,i} = \frac{1}{|\mathcal{N}_{k}|}\sum_{\ell \in \mathcal{N}_{k}} \bphi_{\ell,i}.
\end{equation}
The sample Mean is non-robust with a breakdown point of $0$, hence, a single malicious sample can cause an arbitrary aggregation result. This can be seen in \cref{fig:sc1,fig:2d-sc-mean} with an infinite \gls{ges} and rejection point. It can also be seen as an M-estimator with $\psi(t) = \frac{1}{2}$ in Equation~\eqref{eqn:mest}.

\emph{Coordinate-wise $\alpha$-Trimmed-Mean:} Removes the $\alpha \cdot |\mathcal{N}_{k}|$ smallest and largest values in each dimension and takes the average of the remaining values \cite{Yin.2018, Yang.2019}. Therefore, it has a breakdown point of $\alpha$. The 1D and 2D-\glspl{sc} are shown in \cref{fig:sc1,fig:2d-sc-trimm}. \cref{fig:sc-trimm} shows the \glspl{sc} of the $\alpha$-trimmed mean for varying rates of contamination. As long as the amount of outliers is below the trimmed proportion, the aggregator has a bounded \gls{ges}, but when the amount of outliers exceeds the trimmed proportion, we can observe a linear increase in the \gls{sc}, similar to the non-robust Sample Mean.

\emph{Coordinate-wise Median:} Calculates the value for which the same amount of values are smaller and larger for each dimension \cite{Yin.2018}. It has a breakdown point of $0.5$. The 1D and 2D-\glspl{sc} are shown in \cref{fig:sc1,fig:2d-sc-median}, where we can observe a finite \gls{ges}, but also a jump in the local-shift sensitivity, which explains the lower efficiency of the median.

\emph{(Coordinate-wise) M-Estimators:} For a multivariate estimate of location, M-estimators solve \cite{Schroth.2021, Maronna.2019} 
\begin{align}
	\sum_{\ell \in \mathcal{N}_{k}} 2 \psi(t_{\ell k, i}) (\bphi_{\ell,i} - \bw_{k,i} ) = 0
	\label{eqn:mest}
\end{align} 
with the weight function $\psi(\cdot)$ and the squared Mahalanobis distance
\begin{align}
	t_{\ell k, i} = \left(\bphi_{\ell,i} - \bw_{k,i}\right)^{\top} \bS_{k,i}^{-1} \left(\bphi_{\ell,i} - \bw_{k,i}\right)
\end{align} 
where $\bS_{k,i}$ is a robust scatter matrix estimate. An M-estimator is called monotone when $\sqrt{t} \cdot \psi(t)$ is bounded and $t \cdot \psi(t)$ is non-decreasing, and it is called redescending if $t \cdot \psi(t)$ is redescending. \cref{tb:psi} presents some widespread weight functions, more weight functions can be found in \cite{Schroth.2021, DeMenezes.2021}. The above multivariate M-estimator can also be applied coordinate-wise, i.e. in \cite{Schroth.2023a, Vlaski.2022} coordinate-wise M-estimators were used for robust and efficient aggregation. 

1D and 2D-\glspl{sc} are shown in \cref{fig:sc2,fig:2d-sc-3}. E.g. for Tukey and Talwar, we can observe a finite \gls{ges} and rejection point, with a linear local shift sensitivity, for Huber a finite \gls{ges} and also a linear local shift sensitivity.

\begin{table}[t]
	\centering
	\begin{tblr}{colspec={cccc}, rowsep=5pt, 
		}
		\toprule
		monotone & t & Huber \\
		\midrule
		$\psi(t)$
		& $\frac{\nu + r}{2(\nu + t)}$ 
		& $	\begin{cases}
			\frac{1}{2b} &\\
			\frac{c^{2}}{2bt} &
		\end{cases}$\\
		\midrule
		redescending & Tukey & Talwar   \\
		\midrule
		$\psi(t)$
		& $\begin{cases}
			\frac{t^{2}}{2c^{4}} - \frac{t}{c^{2}} + \frac{1}{2}, & t \leq c^{2}\\
			0, & t > c^{2}
		\end{cases}$ 
		& $\begin{cases}
			\frac{1}{2}, & t \leq c^{2}\\
			0, & t > c^{2}
		\end{cases}$ \\
		\bottomrule
	\end{tblr}
	\caption{$\psi(t)$ functions for M-estimators. For details on the variable $b$ in Huber, see \cite{Schroth.2021}.}
	\label{tb:psi}
\end{table}

\begin{figure}
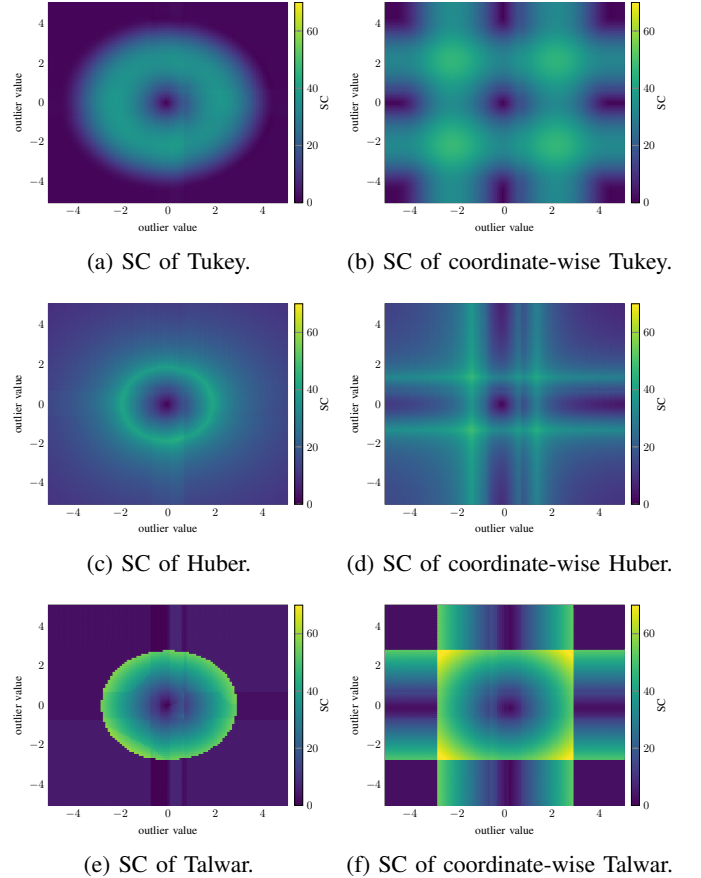

	\centering
	\hfil
	\subcaptionbox{\gls{sc} of Tukey.
		\label{fig:2d-sc-tukey}}[.49\columnwidth]{\resizebox{0.49\columnwidth}{!}{\input{figs/2D-SC-Tukey.tex}}}
	\hfil
	\subcaptionbox{\gls{sc} of coordinate-wise Tukey.
		\label{fig:2d-sc-tukey-coord}}[.49\columnwidth]{\resizebox{0.49\columnwidth}{!}{\input{figs/2D-SC-Tukey-coord.tex}}}
	\hfil \\
	\hfil
	\subcaptionbox{\gls{sc} of Huber.
		\label{fig:2d-sc-huber}}[.49\columnwidth]{\resizebox{0.49\columnwidth}{!}{\input{figs/2D-SC-Huber.tex}}}
	\hfil
	\subcaptionbox{\gls{sc} of coordinate-wise Huber.
		\label{fig:2d-sc-huber-coord}}[.49\columnwidth]{\resizebox{0.49\columnwidth}{!}{\input{figs/2D-SC-Huber-coord.tex}}}
	\hfil \\
	\hfil
	\subcaptionbox{\gls{sc} of Talwar.
		\label{fig:2d-sc-talwar}}[.49\columnwidth]{\resizebox{0.49\columnwidth}{!}{\input{figs/2D-SC-Talwar.tex}}}
	\hfil
	\subcaptionbox{\gls{sc} of coordinate-wise Talwar.
		\label{fig:2d-sc-talwar-coord}}[.49\columnwidth]{\resizebox{0.49\columnwidth}{!}{\input{figs/2D-SC-Talwar-coord.tex}}}
	\hfil
	\caption{Euclidean norm of 2D-\glspl{sc} for M-estimation based aggregation schemes. Values larger than 70 are clipped. Mean of underlying data is at $(0,0)$.}
	\label{fig:2d-sc-3}
\end{figure}

\emph{Geometric Median / RFA:} The geometric median \cite{Weiszfeld.2009, Beck.2015}, also called \gls{rfa} by \cite{Pillutla.2022}, can be seen as a special case of a \gls{mpe} M-estimator with $\psi(t) = \frac{1}{2} \beta t^{\beta-1}$, $\beta=\frac{1}{2}$ and $\bS_{k,i} = \bI$, hence, it solves
\begin{align}
	\sum_{\ell \in \mathcal{N}_{k}} \frac{\bphi_{\ell,i} - \bw_{k,i}}{\lVert\bphi_{\ell,i} - \bw_{k,i}\rVert} = 0.
\end{align} 
The breakdown point is $0.5$, which is similar to the breakdown point of the coordinate-wise median \cite{Lopuhaa.1991}. The \gls{sc} is depicted in \cref{fig:2d-sc-geomedian}, which is almost completely flat.

\emph{SCC:} \glsreset{scc}\gls{scc}, also called ClippedGossip \cite{He.2023}, is based on the idea that each honest agent can trust his own local weight vector $\bphi_{k,i}$. Each weight vector which is further away than a threshold value $\tau_{k,i}$, is then clipped to $\tau_{k,i}$. The aggregation step becomes
\begin{align}
	\bw_{k,i} = \sum_{\ell \in \mathcal{N}_{k}} a_{\ell k} \left(\bphi_{k,i} + \text{CLIP}(\bphi_{\ell,i} - \bphi_{k,i}, \tau_{k,i})\right)
	\label{eqn:scc_w}
\end{align} 
with the clipping function
\begin{align}
	\text{CLIP}(\bx, \tau) = \min(1, \tau / \lVert\bx\rVert) \cdot \bx.
\end{align} 
The threshold can be a fixed value for each time step $\tau_{k,i} = \tau_{k}$ or it can be chosen at each time step. The authors of \cite{He.2023} claim that an adaptive $\tau_{k,i}$ shows a better performance in general, but in \cite{Raynal.2023} it is mentioned that an adaptive $\tau_{k,i}$ can be manipulated by byzantine agents to increase the attack surface. 
Two exemplary 1D-\glspl{sc} are shown in \cref{fig:SC} and two exemplary 2D-\glspl{sc} are shown in \cref{fig:2d-sc-scc1,fig:2d-sc-scc2}. In the 1D-case, the \gls{sc} is very similar to the $\alpha$-trimmed mean with a finite \gls{ges} and a linear local-shift-sensitivity. For the 2D-case, the \gls{sc} is not symmetric, because the clipping is based on a trusted center which might not align with the true underlying mean. Therefore, the influence of an outlier depends on its position with respect to the trusted center and the true mean.

\emph{Krum / Multi-Krum:} Krum calculates the weight vector which has the smallest distance to the $|\mathcal{N}_{k}| - |\mathcal{B}_{k}| - 2$ closest vectors \cite{Blanchard.2017}, given by
\begin{equation}
	\bw_{k,i} = \argmin_{\bphi_{\ell,i}, \ell \in \mathcal{N}_{k}}  \min_{\substack{\mathcal{S}_{k} \subset \mathcal{N}_{k}, \\ |\mathcal{S}_{k}| = |\mathcal{N}_{k}| - |\mathcal{B}_{k}| - 2}} \sum_{s \in \mathcal{S}_{k}} \lVert \bphi_{\ell,i} - \bphi_{s,i}\rVert^{2}.
\end{equation}
To create the set $\mathcal{S}_{k}$, Krum is required to have some knowledge over the number of byzantine neighbors $|\mathcal{B}_{k}|$. Multi-Krum takes the average of the $m$ vectors with the smallest distance. For $m=1$, Krum is obtained, in \cite{Blanchard.2017} $m = |\mathcal{N}_{k}| - |\mathcal{B}_{k}|$ is suggested, this should lead to the highest efficiency, at the cost of a reduced robustness, as the average of all remaining weight vectors is taken.

The \glspl{sc} of Multi-Krum are very similar to \gls{ios} and Talwar as shown in \cref{fig:sc2,fig:2d-sc-krum}. As Multi-Krum calculates the aggregation result by including or discarding individual samples from the averaging step, the \gls{sc} exhibits a stepped appearance based on the underlying data samples.

\emph{IOS / FABA:} IOS \cite{Wu.2023} iteratively discards a total of $|\mathcal{B}_{k}|$ weight vectors, hence, requires the knowledge of $|\mathcal{B}_{k}|$. First it defines a trusted set $\mathcal{T}_{k} = \mathcal{N}_{k}$, then calculates the weighted average of all weight vectors in the trusted set
\begin{equation}
	\bphi_{k,i}^{\text{avg}} =  \frac{\sum_{\ell \in \mathcal{T}_{k}} a_{\ell k} \bphi_{\ell,i}}{\sum_{\ell \in \mathcal{T}_{k}} a_{\ell k}},
\end{equation}
finds the weight vector which has the largest distance to the average (excluding its own value)
\begin{equation}
	j =  \argmax_{\ell \in \mathcal{T}_{k} \backslash \{k\}} \lVert \bphi_{\ell,i} - \bphi_{k,i}^{\text{avg}} \rVert
\end{equation}
and removes this weight vector from the trusted set
\begin{equation}
	\mathcal{T}_{k} = \mathcal{T}_{k} \backslash \{j\}.
\end{equation}
These steps are repeated until $|\mathcal{B}_{k}|$ vectors are removed. Finally, the weighted average of the remaining weight vectors in the trusted set is calculated 
\begin{equation}
	\bw_{k,i} =  \frac{\sum_{\ell \in \mathcal{T}_{k}} a_{\ell k} \bphi_{\ell,i}}{\sum_{\ell \in \mathcal{T}_{k}} a_{\ell k}}.
\end{equation}
\gls{ios} is based on a doubly-stochastic weight matrix $\bA$, e.g. Metropolis rule from Equation~\eqref{eqn:metropolis}, which can be difficult to be calculated reliably in an adversarial network, as it requires the collaboration of the neighboring agents. A simplified version of \gls{ios} is called \gls{faba} \cite{Xia.2019}, which calculates the average with the uniform averaging rule from Equation~\eqref{eqn:uni_avg} and does not require the collaboration of the neighboring agents. The \glspl{sc} are shown in \cref{fig:sc2,fig:2d-sc-ios}, these \glspl{sc} are very similar to Multi-Krum, because both aggregators are based on the same idea of discarding samples which are the furthest away/apart.

\emph{MixTailor:} MixTailor randomly selects in each aggregation round an aggregation scheme from a predefined set of aggregation schemes \cite{Ramezani-Kebrya.2022}. The authors claim that this will make it more difficult for an attacker to develop a tailored attack for a specific aggregation scheme, hence, increasing the robustness against well crafted attacks. \cite{Nabavirazavi.2024} presents a similar idea, but only randomly switches between two different aggregation schemes. In \cref{fig:sc-mix} some exemplary MixTailor variants are presented. In theory, it sounds promising to randomly select an aggregation scheme per aggregation round, but from the \gls{sc} it becomes clear, that MixTailor is only as good as its worst aggregator. This is the case as the \gls{sc} of MixTailor is the average of all its \glspl{sc} and therefore a robust MixTailor is only achieved by selecting aggregators of similar robustness. For example, a non-robust MixTailor is obtained when selecting Talwar and Mean, whereas a robust MixTailor can be obtained for Median and Tukey, as shown in \cref{fig:sc-mix}.

\begin{figure}
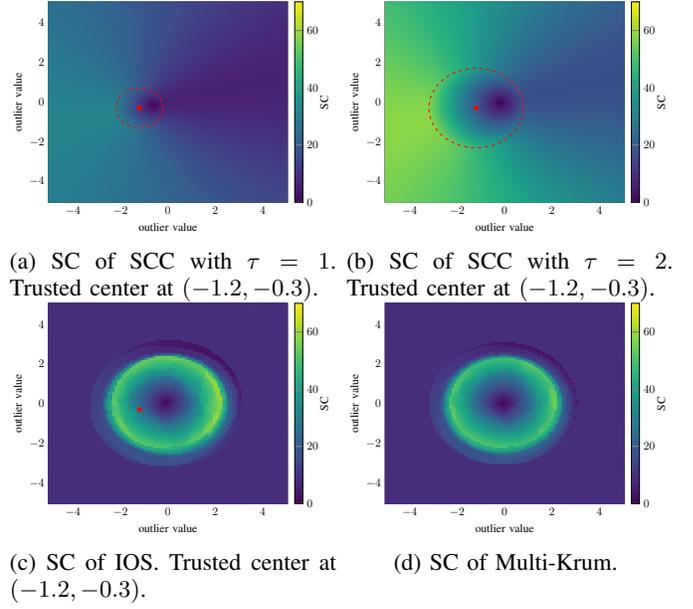

	\centering
	\hfil
	\subcaptionbox{\gls{sc} of \gls{scc} with $\tau = 1$. Trusted center at $(-1.2, -0.3)$.
		\label{fig:2d-sc-scc1}}[.49\columnwidth]{\resizebox{0.49\columnwidth}{!}{\input{figs/2D-SC-SCC-1.tex}}}
	\hfil
	\subcaptionbox{\gls{sc} of \gls{scc} with $\tau = 2$. Trusted center at $(-1.2, -0.3)$.
		\label{fig:2d-sc-scc2}}[.49\columnwidth]{\resizebox{0.49\columnwidth}{!}{\input{figs/2D-SC-SCC-2.tex}}}
	\hfil \\
	\hfil
	\subcaptionbox{\gls{sc} of \gls{ios}. Trusted center at $(-1.2, -0.3)$.
		\label{fig:2d-sc-ios}}[.49\columnwidth]{\resizebox{0.49\columnwidth}{!}{\input{figs/2D-SC-IOS.tex}}}
	\hfil
	\subcaptionbox{\gls{sc} of Multi-Krum.
		\label{fig:2d-sc-krum}}[.49\columnwidth]{\resizebox{0.49\columnwidth}{!}{\input{figs/2D-SC-MultiKrum.tex}}}
	\hfil
	\caption{Euclidean norm of 2D-\glspl{sc} for different aggregation schemes. Values larger than 70 are clipped. Mean of underlying data is at $(0,0)$.}
	\label{fig:2d-sc-2}
\end{figure}

\section{Attack Schemes}
\label{sec:attacks}

In this section, we will give an overview of existing attack schemes. In the \emph{Gaussian attack}, the attacker draws Gaussian distributed noise with varying variance and either sends this sample to the attacked agent \cite{Li.2023, Ghavamipour.2024}.
In \emph{Sign-flipping}, the sign of the weight vector is flipped \cite{Karimireddy.2021}. Originally, developed for gradient aggregation, this would lead to a maximization of the loss. \emph{Label-flipping} is specifically designed for classification problems, e.g. MNIST dataset or CIFAR-10 dataset, where labels of training data are flipped \cite{Yin.2018, Ghavamipour.2024}. This attack requires full access to the training process, as the underlying training data is manipulated and not only the received weight update. The \emph{MIMIC attack} or \emph{Sample-duplicating attack} copies the weight vector of the attacked agent, which aims to over-emphasize the attacked agent \cite{Karimireddy.2022, Dong.2024}. Lastly, in \gls{ipm}, the inner product between the robust estimation and the true gradient is manipulated in such a way to be negative, which aims to hinder convergence \cite{Xie.2020}. In what follows, we present attacks, which we will use as reference.

Agent $j \in \mathcal{B}$ is the attacking byzantine agent, while agent $k$ is attacked and $j \in \mathcal{B}_{k}$. Each byzantine agent $j$ crafts a specific attack for each honest neighbor $k$, hence, the attack vector from agent $j$ to agent $k$ at time $i$ is denoted by $\bz_{jk,i}$.

\emph{Large Value / Scaling Attack:} In the \gls{lv} or scaling attack, the attacker sends model weights with a large magnitude \cite{Schroth.2023, Ghavamipour.2024, Bagdasaryan.2020}. Some exemplary implementations include a scaled version of the attackers own model
\begin{equation}
	\bz_{jk,i} =  \gamma \cdot \bphi_{j,i},
\end{equation}
a scaled (inverse) version of the honest update of the attacked agent $k$
\begin{equation}
	\bz_{jk,i} = \pm \gamma \cdot \sum_{\ell \in \mathcal{H}_{k}} a_{\ell k} \bphi_{\ell,i},
	\label{eqn:lv_posneg}
\end{equation}
a scaled vector of ones
\begin{equation}
	\bz_{jk,i} =  \gamma \cdot \bone,
	\label{eqn:lv_ones}
\end{equation}
or a scaled randomly generated vector.

\emph{ALIE:} In \gls{alie}, the cumulative standard normal distribution $\phi(z)$ is used to determine the value $z^{\text{max}}$ such that at least $s$ honest agents are further away from the mean value than byzantine agents \cite{Baruch.2019}. First we calculate
\begin{equation}
	s_{jk} = \left\lfloor \frac{|\mathcal{N}_{k}|}{2} + 1\right\rfloor - |\mathcal{B}_{k}|
\end{equation}
and
\begin{equation}
	z^{\text{max}}_{jk} = \max_{z} \left(\phi(z) < \frac{|\mathcal{N}_{k}| - |\mathcal{B}_{k}| - s_{jk}}{|\mathcal{N}_{k}| - |\mathcal{B}_{k}|}\right).
\end{equation}
Subsequently, the coordinate-wise mean and standard deviation over all honest neighbors $\mathcal{H}_{k}$ are calculated and collected in the vectors $\bmu_{k,i}^{\mathcal{H}}$ and $\bsigma_{k,i}^{\mathcal{H}}$, respectively. The final outlier is then obtained by
\begin{equation}
	\bz_{jk,i} = \bmu_{k,i}^{\mathcal{H}} - z^{\text{max}}_{jk} \cdot \bsigma_{k,i}^{\mathcal{H}}.
\end{equation}

\emph{ROP:} In \gls{rop}, it is the goal to find an attack vector which is orthogonal towards the honest update direction and relocated around a reference point \cite{Ozfatura.2024}. In particular, \gls{rop} was developed to circumvent \gls{scc}. Hence, the reference point is assumed to be weight vector $\bphi_{k,i}$ of the attacked agent $k$, but could be adapted to any other reference point. First, the honest update direction is calculated as
\begin{equation}
	\bDelta_{jk,i}^{\mathcal{H}} = \frac{1}{|\mathcal{H}_{k}|}\sum_{\ell \in \mathcal{H}_{k}} \bphi_{\ell,i} - \bphi_{k,i},
	\label{eqn:ropref}
\end{equation}
followed by the orthogonal vector to honest update direction
\begin{equation}
	\bp_{jk,i} = \bone - \frac{\bone^{\top}  \bDelta_{jk,i} ^{\mathcal{H}}}{\lVert \bDelta_{jk,i}^{\mathcal{H}}\rVert^{2}} \cdot \bDelta_{jk,i}^{\mathcal{H}}.
\end{equation}
In Equation~\eqref{eqn:ropref}, we have adapted to $\bphi_{k,i}$ as \gls{scc} clips around this value in Equation~\eqref{eqn:scc_w}. According to \cite{Ozfatura.2024}, the previous weight vector $\bw_{k,i-1}$ would also be a valid choice.

To increase flexibility, the attack vector can be rotated to any angle $\theta$ and scaled with parameter $\gamma$. Finally, the attack vector is relocated around the reference point
\begin{equation}
	\bz_{jk,i} = \gamma \left(\sin(\theta) \frac{\bp_{jk,i}}{\lVert \bp_{jk,i}\rVert} + \cos(\theta) \frac{\bDelta_{jk,i}^{\mathcal{H}}}{\lVert\bDelta_{jk,i}^{\mathcal{H}}\rVert}\right) + \bphi_{k,i}.
\end{equation}
The authors in \cite{Ozfatura.2024} suggest $\theta = \frac{\pi}{2}$ and $\gamma = 1$, when attacking \gls{scc} with $\tau_{k,i} = \{0.1, 1\}$.

\section{SASCM Attack Examples}
\label{sec:example_attacks}

\emph{Coordinate-wise M-estimators:}
Firstly, the estimator specific constant which is based on the tuning parameter of the specific M-estimator has to be determined. For monotone and redescending M-estimators $\sqrt{t} \cdot \psi(t)$ will be always bounded, hence, the constant $z_{0}$ can be found by solving
\begin{equation}
	z_{0}^{2} = \argmax_{t} \left| \sqrt{t} \cdot \psi(t) \right|.
\end{equation}
Results for some exemplary M-estimators are depicted in \cref{tb:SC_m_est}. As the coordinate-wise M-estimator performs a coordinate-wise aggregation, the \gls{scm} vector can also be calculated in a coordinate-wise manner. The initial outlier values for each dimension $m$ are calculated by inverting the initial normalization step performed by the M-estimator, i.e., 
\begin{equation}
	z_{m}^{\textrm{init}} = z_0 \cdot \textrm{mad}(\mathcal{Y}_{m}) + \textrm{median}(\mathcal{Y}_{m})
\end{equation}
where $\textrm{mad}(\cdot)$ denotes the median absolute deviation. To obtain the final outlier value, the initial outlier values have to be added to the honest data as $\mathcal{Z}_{m}^{\textrm{init}} = \{z_{m}^{\textrm{init}} \cdot \mathds{1}_{P}\}$, resulting in
\begin{equation}
	z_{m}^{\star} = z_0 \cdot \textrm{mad}(\mathcal{Y}_{m} \cup \mathcal{Z}_{m}^{\textrm{init}}) + \textrm{median}(\mathcal{Y}_{m} \cup \mathcal{Z}_{m}^{\textrm{init}}).
\end{equation}
Collecting $z_{m}^{\star}$ for each dimension into a vector, we obtain the \gls{scm} vector $\bz^{\star}$. As the \gls{sc} curves for M-estimators are symmetric, there exist two solutions per dimension, i.e. $z_{m}^{\star}$ and $- z_{m}^{\star}$, hence, in total $2^r$ candidate vectors exist. For each iteration the same candidate vector should be selected, to allow for an accumulation of the attack.

\begin{table}[t]
	\centering
	\begin{tblr}{colspec={lcccc}, rowsep=5pt}
		\toprule
		& t & Huber & Tukey & Talwar\\
		\midrule
		$z_{0}^{2}$ & $\nu$	& $c^{2}$ & $\frac{c^{2}}{5}$ & $c^{2}$ \\
		\bottomrule
	\end{tblr}
	\caption{Values of $z_{0}^{2}$ to calculate the optimal outlier for M-estimators.}
	\label{tb:SC_m_est}
\end{table}

\emph{Multivariate M-estimators:}
Again, the initial normalization step done by the M-estimator has to be inverted. The scatter matrix is approximated by a diagonal matrix with the median absolute deviations for each dimension on the diagonal,
\begin{equation}
	\bS_{0} = \bI \cdot \textrm{mad}(\mathcal{Y}).
\end{equation}
Then the scatter matrix is decomposed via eigendecomposition for symmetric matrices as
\begin{equation}
	\bS_{0} = \bV \bLambda \bV^{\top} = \left(\bV \sqrt{\bLambda} \bV^{\top}\right)^{2}
\end{equation}
where the columns of the orthogonal matrix $\bV$ contain the eigenvectors of $\bS_{0}$ and $\bLambda$ is a diagonal matrix with the eigenvalues of $\bS_{0}$ on its diagonal. The initial outlier can then be calculated by
\begin{equation}
	\bz^{\textrm{init}} = z_0 \cdot \bV \sqrt{\bLambda} \bV^{\top} \bd + \textrm{median}(\mathcal{Y})
\end{equation}
where $\bd$ is a direction vector with $\lVert\bd\rVert = 1$. The initial outlier set can then be created as $\mathcal{Z}^{\textrm{init}} = \{\bz^{\textrm{init}} \cdot \mathds{1}_{P}\}$. Repeating the above steps once with the contaminated set $\{\mathcal{Y} \cup \mathcal{Z}^{\textrm{init}}\}$, leads to the optimal outlier $\bz^{\star}$. To ensure an accumulation of the attack over time, the direction vectors at different time indices have to be positively aligned.

\emph{IOS:} The \gls{scm} attack for \gls{ios} can be approximated by finding the last discarded weight vector from the aggregation procedure. This weight vector could be directly injected into the aggregation procedure, but this would not control the alignment over time. Therefore the attacker should craft an attack vector with the same distance but a controlled direction with regard to the aggregation result.

\emph{SCC:} For the parameter combination $\theta = \pi$ and $\gamma \geq \tau_{k,i}$, \gls{rop} follows \cref{th:SCM} and therefore maximizes the \gls{sc} of \gls{scc}. But the attack will not necessarily accumulate over time as the direction $\bDelta_{jk,i}^{\mathcal{H}}$ in Equation~\eqref{eqn:ropref} does randomly fluctuate. Finding an attack value which also holds for \cref{th:ascm} is not possible, because the value which maximizes the \gls{scm} in each round is not aligned over time. Therefore, we cannot expect \gls{rop} with these parameters to be the most powerful attack.

When using the implementation of \gls{rop} proposed in \cite{Ozfatura.2024}, where $\bphi_{k,i}$ in Equation~\eqref{eqn:ropref} was replaced with $\bw_{k,i-1}$, the direction of $\bDelta_{jk,i}^{\mathcal{H}}$ will not randomly fluctuate. Hence, the attack values are aligned over time, but the attack values will not maximize the \gls{sc}. Using $\theta = \pi$ in this implementation does not change the update direction and only slows down convergence. Therefore the authors of \cite{Ozfatura.2024} suggest $\theta = \frac{\pi}{2}$, which leads to an orthogonal vector to $\bDelta_{jk,i}^{\mathcal{H}}$. This orthogonal vector inflicts a large perturbation on the update direction, which might explain its effectiveness.

Finding an optimal attack for \gls{scc} seems to be non trivial, as fulfilling both theorems exactly is not possible. An optimal attack for \gls{scc} would have to find an attack which trades-off alignment over time, maximizing the \gls{sc} and the amount of change in the update direction. We leave this as an open research question.

\section{Simulations}
\label{sec:simulations}

The simulations are performed on a network with $K=30$ agents, arranged in a Erdős–Rényi graph with an edge probability of 70\%. The following assumptions are ensured for every graph, which are commonly found in the literature:
\begin{itemize}
	\item The majority of the agents in the graph are benign, hence, for the global contamination rate, it holds $\varepsilon < 0.5$.
	\item The majority of each neighborhood $\mathcal{N}_{k}$ is benign, hence, for the local contamination rate, it holds $\varepsilon_{k} < 0.5$.
\end{itemize}

\subsection{Choice of Parameters}

Most aggregators rely on some kind of tuning parameter to trade-off robustness and efficiency. In general we have chosen parameters which are commonly used in the literature and which achieve a good trade-off. We have set $\alpha = 0.0688$ for the coordinate-wise $\alpha$-Trimmed-Mean, $c = 4.685$ for Tukey, $c = 2.7955$ for Talwar \cite{Schroth.2023}, $q = 0.8$ which leads to $c^{2} = F_{\chi_{r}^{2}}^{-1}(q)$ for Huber \cite{Schroth.2021}. For \gls{scc}, we use an adaptive $\tau_{k,i}$ \cite{He.2023} and a fixed $\tau = 0.1$ \cite{Ozfatura.2024}. MixTailor is composed out of Median and coordinate-wise Tukey with the same parameters as above. In \gls{ios} the number of discarded samples is equal to the true number of Byzantine neighbors $|\mathcal{B}_{k}|$. In Multi-Krum $m = |\mathcal{N}_{k}| - |\mathcal{B}_{k}|$ to achieve the highest efficiency \cite{Blanchard.2017}.

For the \gls{lv} attack, the scaling factor is set to $\gamma = 1000$. In \gls{rop}, it is sufficient to set $\gamma \geq \tau$, hence, we set $\gamma = 10$ with the attack angles $\theta = \{\pi, \frac{\pi}{2}\}$. The \gls{sascm} attacks are crafted such that they fulfill \cref{def:sascm}. For direction independent \glspl{sc}, a scaled vector of ones is chosen to achieve alignment over time.

\subsection{Linear Regression}
Each agent $k$ observes a local linear model of the form
\begin{equation}
	\boldsymbol{d}_{k} = \boldsymbol{u}^{\top}_{k} \bw^{\circ} + \boldsymbol{v}_{k}
\end{equation}
with the regressors $\boldsymbol{u}_{k} \in \mathbb{R}^{10 \times 1}$ being independently and identically distributed as $\boldsymbol{u}_{k} \sim \mathcal{N}(0, \bI_{10})$. The noise is distributed as $\boldsymbol{v}_{k} \sim \mathcal{N}(0, \sigma^2_{v}) \in \mathbb{R}^{10 \times 1}$ with $\sigma^2_{v} = 0.01$ and the learning rate is set to $\mu = 0.05$. Each agent employs a Huber loss function which guarantees a finite gradient $||\nabla J_{k}(\bw)|| < \infty$.

\subsection{MNIST Dataset}

To simulate a classification problem, we use the \gls{mnist} dataset \cite{LeCun.1998}. The training and test data is split randomly in chunks of the same size, which are distributed to each agent. The deployed batch size is $32$ with a constant learning rate of $\mu = 0.1$ and a cross-entropy loss function.

\subsection{Results}
\subsubsection{Baseline without attack}
In \cref{fig:sim_no_out}, we establish a base-line training loss for all aggregation schemes on linear regression without outliers. We can observe the fastest convergence speed for Mean aggregation, closely followed by \gls{scc}, Multi-Krum, \gls{ios} and $\alpha$-Trimmed-Mean. The scenario without aggregation also converges fast, but cannot reach the loss level of the other aggregation schemes. These results indicate the performance gain by using distributed algorithms. We notice that the multivariate M-estimators Talwar and Tukey, take much longer to converge. As the number of samples per aggregation step in comparison to the dimension is relatively low, we assume that there could be some dimensionality issues. All other aggregators take roughly the same time, with the coordinate-wise M-estimators being faster than the Median and Geometric Median, which can be attributed to the higher efficiency of the coordinate-wise M-estimators. This results suggest that the deployment of coordinate-wise M-estimators is more favorable than multivariate M-estimators.

For the \gls{mnist} data, the baseline accuracy without outliers is shown in \cref{fig:sim_mnist_no_out}. The highest accuracies can be observed for mean, \gls{scc}, Multi-Krum, \gls{ios} and $\alpha$-Trimmed-Mean at around $98\%$. Median and Geometric Median exhibit a slightly lower accuracy at around $96\%$. We cannot observe a large discrepancy between coordinate-wise and multivariate M-estimators, as observed previously for linear regression. The multivariate Huber M-estimator exhibits the best accuracy with around $98\%$ and the multivariate Talwar M-estimator the worst accuracy with about $96\%$, with the remaining M-estimators in-between.

\begin{figure}
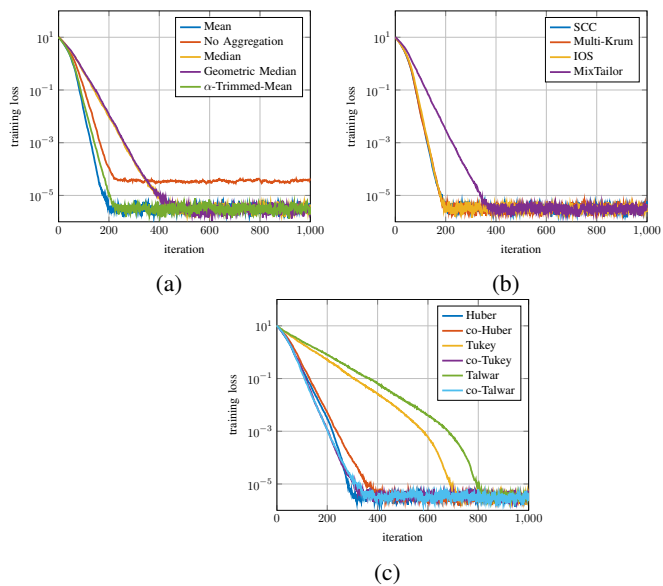

	\centering
	\hfil
	\subcaptionbox{%
		\label{fig:sim_no_out_1}}[.49\columnwidth]{\resizebox{0.49\columnwidth}{!}{\input{figs/sim-linear-no_out-1.tex}}}
	\hfil
	\subcaptionbox{%
		\label{fig:sim_no_out_2}}[.49\columnwidth]{\resizebox{0.49\columnwidth}{!}{\input{figs/sim-linear-no_out-2.tex}}}
	\hfil
	\\
	\hfil
	\subcaptionbox{%
		\label{fig:sim_no_out_3}}[.49\columnwidth]{\resizebox{0.49\columnwidth}{!}{\input{figs/sim-linear-no_out-3.tex}}}
	\hfil
	\caption{Training loss for all aggregation schemes without malicious agents for linear regression.}
	\label{fig:sim_no_out}
\end{figure}

\begin{figure}
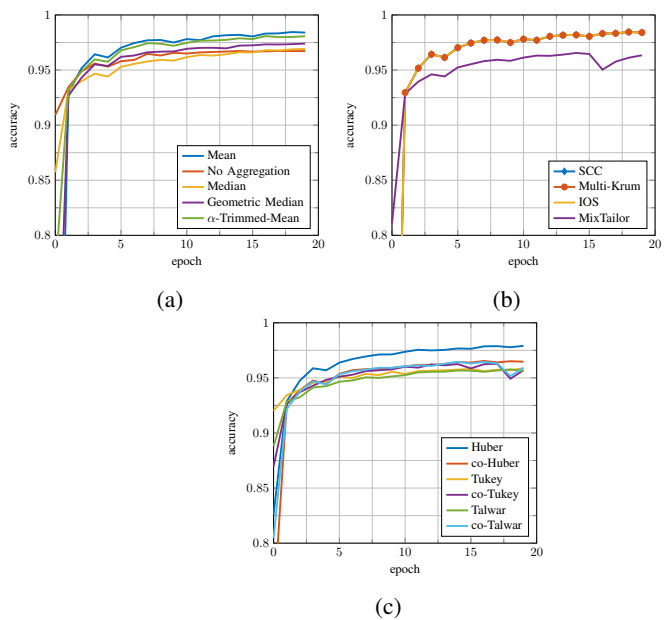

	\centering
	\hfil
	\subcaptionbox{%
		\label{fig:sim_mnist_no_out_1}}[.49\columnwidth]{\resizebox{0.49\columnwidth}{!}{\input{figs/sim-mnist-no_out-1.tex}}}
	\hfil
	\subcaptionbox{%
		\label{fig:sim_mnist_no_out_2}}[.49\columnwidth]{\resizebox{0.49\columnwidth}{!}{\input{figs/sim-mnist-no_out-2.tex}}}
	\hfil
	\\
	\hfil
	\subcaptionbox{%
		\label{fig:sim_mnist_no_out_3}}[.49\columnwidth]{\resizebox{0.49\columnwidth}{!}{\input{figs/sim-mnist-no_out-3.tex}}}
	\hfil
	\caption{Accuracy for all aggregation schemes without malicious agents for \gls{mnist} data.}
	\label{fig:sim_mnist_no_out}
\end{figure}

\subsubsection{Analysis of robust aggregators under SASCM attack}
We compare our proposed \gls{sascm} attack with the \gls{lv} and \gls{alie} attack, which are well-known in the literature. The details of the attacks are given in \cref{sec:attacks}. If not stated otherwise, the \gls{lv} attack was implemented according to \cref{eqn:lv_ones}. In \cref{fig:sim_ios,fig:sim_huber}, the results of attacking \gls{ios} and the (coordinate-wise) Huber M-estimator are shown. It can be observed, that the \gls{sascm} attack inflicts the largest perturbation on the training loss, whereas the influence of the \gls{lv} and \gls{alie} attacks are low. When attacking \gls{scc}, we will use \gls{rop}, because based on the implementation and chosen angle, it can be used to approximate an \gls{sascm} attack. \gls{rop} from Equation~\eqref{eqn:ropref} with $\theta = \pi$ maximizes the \gls{sc}, therefore is an \gls{scm} attack, but it is less strict on aligning the attack over time. Whereas \gls{rop} from \cite{Ozfatura.2024} has a better alignment over time, but does only approximately maximize the \gls{sc}. In \cref{fig:sim_scc}, the largest influence can be observed by the \gls{lv} attack, which can be explained by the fact that an adaptive $\tau$ is used, which is more susceptible to large attack values. Good results can also be observed for \gls{rop} with $\theta = \frac{\pi}{2}$. \gls{rop} with $\theta = \pi$ shows a worse performance because it can be seen as only decreasing the learning rate and not introducing a perturbation.

In \cref{fig:sim_mnist_ios}, \gls{ios} and in \cref{fig:sim_mnist_huber}, the (coordinate-wise) Huber M-estimator based on the \gls{mnist} dataset are attacked. In both figures it can be clearly observed that \gls{sascm} inflicts the largest reduction in accuracy. For the (coordinate-wise) Huber M-estimator in \cref{fig:sim_mnist_huber_6}, \gls{sascm} is able to reduce the accuracy to random guessing.

\begin{figure}
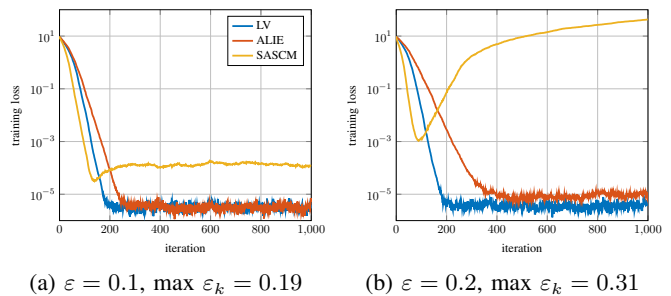

	\centering
	\hfil
	\subcaptionbox{$\varepsilon = 0.1$, max $\varepsilon_{k} = 0.19$ \label{fig:sim_ios_3}}[.49\columnwidth]{\resizebox{0.49\columnwidth}{!}{\input{figs/sim-linear-ios-3.tex}}}
	\hfil
	\subcaptionbox{$\varepsilon = 0.2$, max $\varepsilon_{k} = 0.31$
		\label{fig:sim_ios_6}}[.49\columnwidth]{\resizebox{0.49\columnwidth}{!}{\input{figs/sim-linear-ios-6.tex}}}
	\hfil
	\caption{\gls{lv}, \gls{alie} and \gls{sascm} attack on \gls{ios} for linear regression.}
	\label{fig:sim_ios}
\end{figure}

\begin{figure}
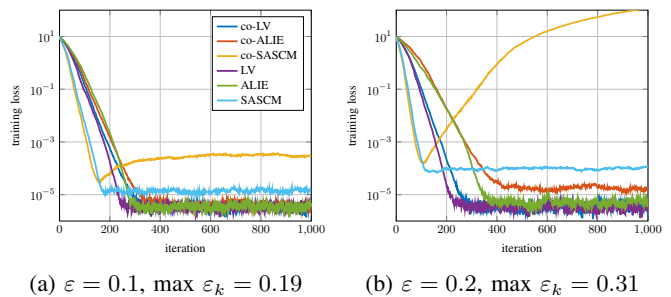

	\centering
	\hfil
	\subcaptionbox{$\varepsilon = 0.1$, max $\varepsilon_{k} = 0.19$ \label{fig:sim_huber_3}}[.49\columnwidth]{\resizebox{0.49\columnwidth}{!}{\input{figs/sim-linear-huber-3.tex}}}
	\hfil
	\subcaptionbox{$\varepsilon = 0.2$, max $\varepsilon_{k} = 0.31$
		\label{fig:sim_huber_6}}[.49\columnwidth]{\resizebox{0.49\columnwidth}{!}{\input{figs/sim-linear-huber-6.tex}}}
	\hfil
	\caption{\gls{lv}, \gls{alie} and \gls{sascm} attack on (coordinate-wise) Huber M-estimator for linear regression.}
	\label{fig:sim_huber}
\end{figure}

\begin{figure}
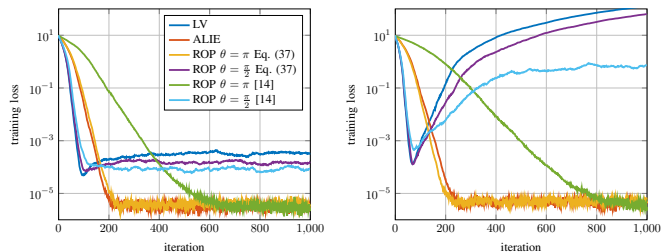

	\centering
	\hfil
	\subcaptionbox{$\varepsilon = 0.1$, max $\varepsilon_{k} = 0.19$ \label{fig:sim_scc_3}}[.49\columnwidth]{\resizebox{0.49\columnwidth}{!}{\input{figs/sim-linear-scc-3.tex}}}
	\hfil
	\subcaptionbox{$\varepsilon = 0.2$, max $\varepsilon_{k} = 0.31$
		\label{fig:sim_scc_6}}[.49\columnwidth]{\resizebox{0.49\columnwidth}{!}{\input{figs/sim-linear-scc-6.tex}}}
	\hfil
	\caption{\gls{lv}, \gls{alie} and \gls{rop} attack on \gls{scc} with an adaptive $\tau_{k,i}$ for linear regression. \gls{rop} from Equation~\eqref{eqn:ropref} with $\theta = \pi$ maximizes the \gls{sc}.}
	\label{fig:sim_scc}
\end{figure}

\begin{figure}
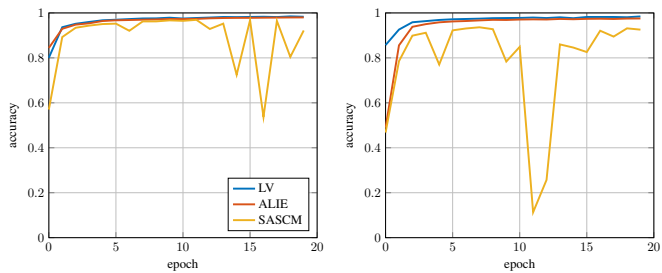

	\centering
	\hfil
	\subcaptionbox{$\varepsilon = 0.1$, max $\varepsilon_{k} = 0.19$ \label{fig:sim_mnist_ios_3}}[.49\columnwidth]{\resizebox{0.49\columnwidth}{!}{\input{figs/sim-mnist-ios-3.tex}}}
	\hfil
	\subcaptionbox{$\varepsilon = 0.2$, max $\varepsilon_{k} = 0.31$
		\label{fig:sim_mnist_ios_6}}[.49\columnwidth]{\resizebox{0.49\columnwidth}{!}{\input{figs/sim-mnist-ios-6.tex}}}
	\hfil
	\caption{\gls{lv}, \gls{alie} and \gls{sascm} attack on \gls{ios} for \gls{mnist} data.}
	\label{fig:sim_mnist_ios}
\end{figure}

\begin{figure}
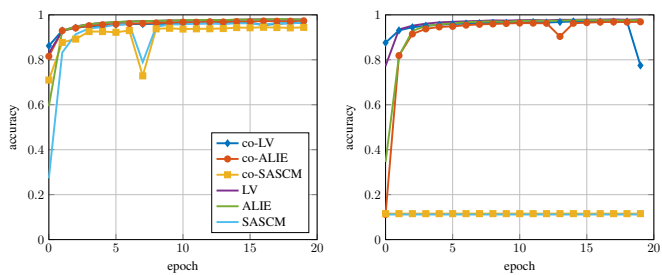

	\centering
	\hfil
	\subcaptionbox{$\varepsilon = 0.1$, max $\varepsilon_{k} = 0.19$ \label{fig:sim_mnist_huber_3}}[.49\columnwidth]{\resizebox{0.49\columnwidth}{!}{\input{figs/sim-mnist-huber-3.tex}}}
	\hfil
	\subcaptionbox{$\varepsilon = 0.2$, max $\varepsilon_{k} = 0.31$
		\label{fig:sim_mnist_huber_6}}[.49\columnwidth]{\resizebox{0.49\columnwidth}{!}{\input{figs/sim-mnist-huber-6.tex}}}
	\hfil
	\caption{\gls{lv}, \gls{alie} and \gls{sascm} attack on (coordinate-wise) Huber M-estimator for \gls{mnist} data.}
	\label{fig:sim_mnist_huber}
\end{figure}

\subsubsection{Analysis of different alignment strategies}
In \cref{fig:sim_ones}, we simulate an \gls{lv} attack on a Median aggregator with different alignment strategies as shown in \cref{sec:attacks}. When choosing a random direction every time step, we can observe the smallest divergence of the training loss. This result is expected because in this scenario the attack perturbations cannot accumulate over time. Choosing the inverse honest update direction, leads to significant slow-down of the learning procedure for small contamination rates and a complete stop in learning for larger contamination rates.

\begin{figure}
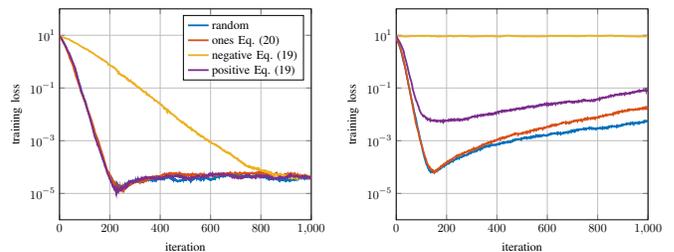

	\centering
	\hfil
	\subcaptionbox{$\varepsilon = 0.1$, max $\varepsilon_{k} = 0.19$ \label{fig:sim_ones_3}}[.49\columnwidth]{\resizebox{0.49\columnwidth}{!}{\input{figs/sim-linear-ones-3.tex}}}
	\hfil
	\subcaptionbox{$\varepsilon = 0.2$, max $\varepsilon_{k} = 0.31$
		\label{fig:sim_ones_6}}[.49\columnwidth]{\resizebox{0.49\columnwidth}{!}{\input{figs/sim-linear-ones-6.tex}}}
	\hfil
	\caption{\gls{lv} attacks on Median for linear regression with different alignment strategies.}
	\label{fig:sim_ones}
\end{figure}

\section{Conclusion}
\label{sec:conclusion}
We analyzed robust aggregation schemes using the \gls{sc} and the proposed \gls{scm}, \gls{ascm} and \gls{sascm} to craft powerful attacks for every robust aggregation scheme. The general idea is to find an outlier which maximizes the \gls{sc}, while maintaining a positive alignment over time for the different outliers. The simulations show that the combination of \gls{scm} and accumulation over time, i.e. \gls{ascm} or \gls{sascm}, leads to a breakdown of the considered robust aggregation schemes. In conclusion, developing a robust aggregation scheme which is absolutely robust seems to be very difficult, if not impossible, as every aggregation scheme which attempts to incorporate information from neighboring agents can be affected by the proposed attacks.

\subsection{Future Work}
In \cref{sec:scm}, we have formalized that not only the distance, but also the direction of the outlier is important for the maliciousness of the attack. This has not been unnoticed in literature, e.g. in \cite{Raynal.2024}, it is stated that ``Distance is not a Proxy for Maliciousness''. Especially, for \gls{scc}, we have noticed that the reference point is less important than the position of the outlier with regard to the update direction. Hence, it would be interesting to further investigate the influence of the attack direction on the aggregation result. Additionally, investigating different distance measures in \cref{def:opt_attack} would be interesting.

\printbibliography

\end{document}